\documentclass{article}

\usepackage[preprint]{neurips_2026}

\usepackage[utf8]{inputenc} 
\usepackage[T1]{fontenc}    
\usepackage{hyperref}       
\usepackage{url}            
\usepackage{booktabs}       
\usepackage{amsfonts}       
\usepackage{nicefrac}       
\usepackage{microtype}      
\usepackage{xcolor}         
\usepackage{subcaption}
\usepackage{graphicx}
\usepackage{float}
\usepackage{algorithm}
\usepackage{algorithmic}
\usepackage{amsthm}
\usepackage{multirow}
\usepackage{soul}
\usepackage{bm}
\usepackage{amsmath}

\title{Decomposing A Neural Learner}

\author{%
  Bi, Zelong\\
  Department of Mathematics \& Statistics\\
  University of New South Wales\\
  High St Kensington NSW 2052 Australia \\
  \texttt{zelong.bi@unsw.edu.au} \\
}

\newtheorem{theorem}{Theorem}
\newtheorem*{theorem*}{Theorem}
\newtheorem{remark}{Remark}
\newcommand{\norm}[1]{\left\lVert#1\right\rVert}

\begin{document}

\maketitle

\begin{abstract}
  Neural learners for classification and regression are often criticized as black boxes. We propose a probabilistic decomposition of neural learners that accounts for heterogeneous predictability across the feature space. A training procedure combining the EM algorithm, stochastic gradient descent, and maximum likelihood estimation is developed to jointly learn the predictable component and an input-dependent reliability score. We show that the optimal decomposed learner recovers the response distribution within the predictable component, while the reliability score characterizes the predictability of individual observations. We further extend the framework to incorporate localized feature transformations. Experiments on synthetic and real-world datasets demonstrate the effectiveness of the proposed framework.
\end{abstract}

\section{Introduction}

The success of modern machine learning in fields such as image processing \citep{krizhevsky2012imagenet} and large language models \citep{vaswani2017attention} can be traced back to developments in neural networks in the 1990s, when neural networks in their simplest form, multilayer perceptrons (MLPs), were first used for supervised learning tasks such as regression and classification \cite{murtagh1991multilayer}. From the perspective of the universal approximation theorem \citep{cybenko1989approximation}, if the features and response variable $\bm{x}$ and $y$ are related through an appropriate functional relationship, such that $\mathbb{E}[y\mid \bm{x}]$ or $p(y\mid\bm{x})$ is some function $f(\bm{x})$, then the neural learner can be viewed as learning to approximate the underlying function $f$.

\begin{figure}[H]
    \centering
    \begin{subfigure}[b]{0.3\textwidth}
        \centering
        \includegraphics[height=0.6\textwidth]{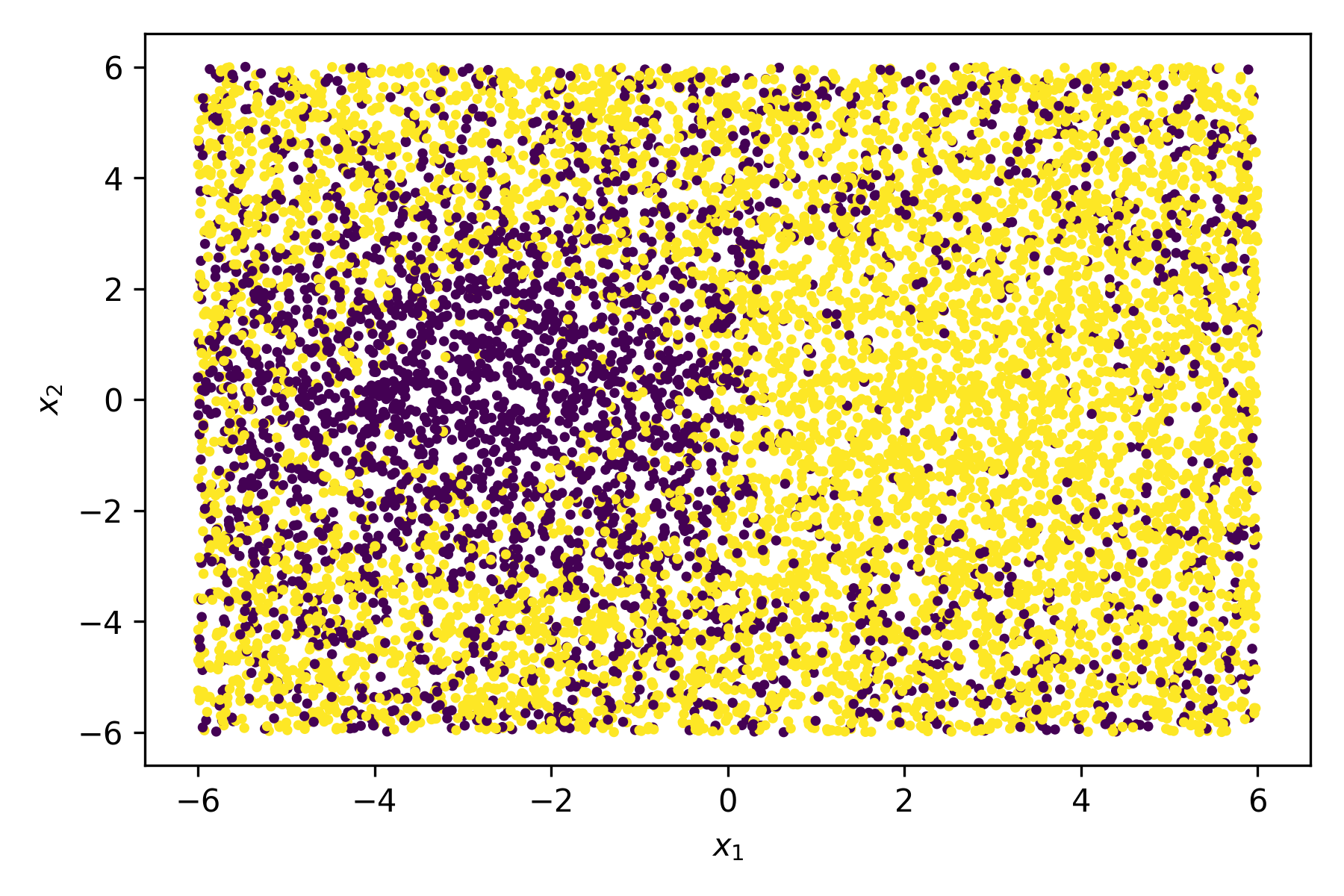}
        \caption{Heterogeneous feature space}
        \label{fig:01a}
    \end{subfigure}
    \hfill
    \begin{subfigure}[b]{0.3\textwidth}
        \centering
        \includegraphics[height=0.6\textwidth]{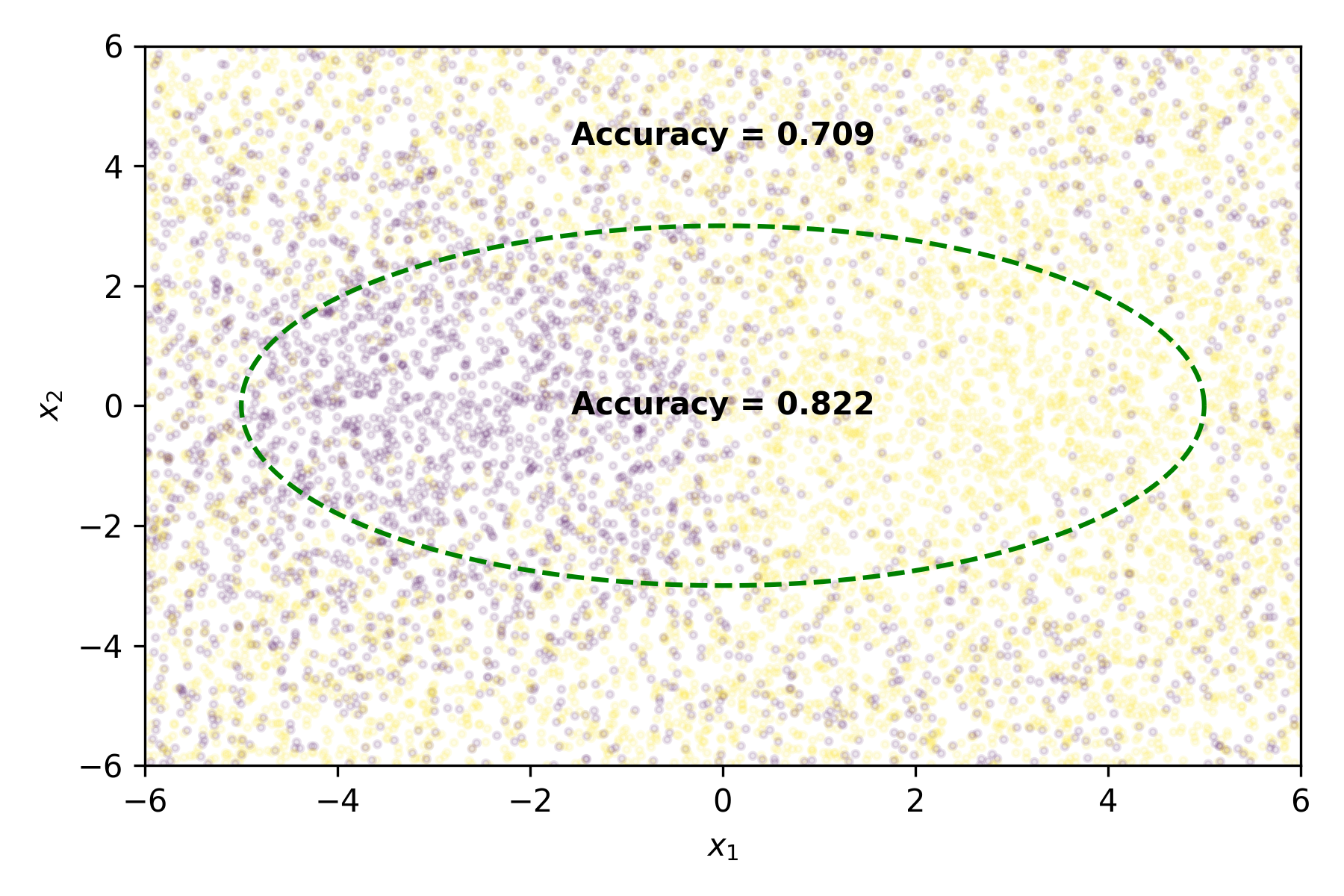}
        \caption{Accuracy by region}
        \label{fig:01b}
    \end{subfigure}
    \hfill
    \begin{subfigure}[b]{0.3\textwidth}
        \centering
        \includegraphics[height=0.6\textwidth]{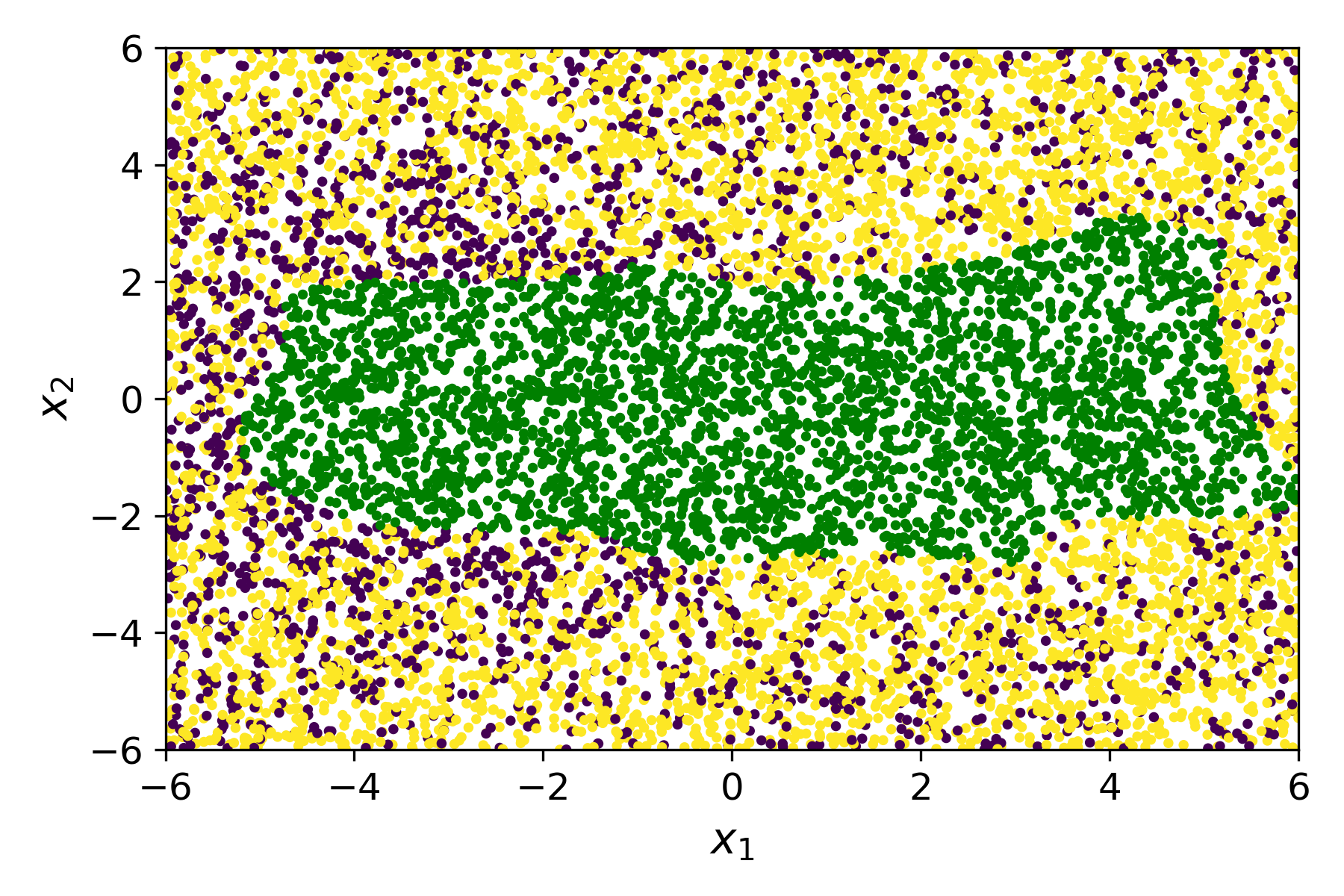}
        \caption{High reliability instances}
        \label{fig:01c}
    \end{subfigure}

        \caption{(a) Labeled dataset across the feature space. (b) Accuracy of a neural classifier in different regions. (c) High-reliability instances (in green, with $p(z=1\mid\bm{x}) > 0.5$) identified by our model.}
    \label{fig:01}
\end{figure}

A common criticism of neural learners is that, despite their strong predictive performance, the structure learned is often difficult to interpret or inspect, contributing to their characterization as black-box models \citep{rudin2019stop}. One aspect of the learned information that is particularly important but largely hidden in the standard formulation is where in the feature space the relationship between $y$ and $\bm{x}$ exists. The feature space need not be homogeneous in this respect: in some regions, $y$ may depend strongly on $\bm{x}$, whereas in others, they can be independent. For example, in Figure~\ref{fig:01a}, the relationship between $\bm{x}$ and $y$ appears strong in the central region but absent elsewhere.

The absence of a relationship is itself informative. When $\bm{x}$ and $y$ are independent in a region, a well-trained neural learner may correctly learn a constant function, but this does not reveal that $\bm{x}$ provides no predictive information or indicate that the resulting prediction is unsupported by the available features. As shown in Figure~\ref{fig:01b}, an overall accuracy of 74\% by a neural classifier can arise from substantially different accuracies across regions, such information is not available in the common way it is trained and used. Explicitly modelling such structure can therefore provide a more informative characterization of the data and improve the transparency of neural learners without sacrificing predictive performance, which can be very useful when assessing their deployment in critical applications.

In this manuscript, we propose a novel framework that improves the interpretability of neural learners by decomposing its predication. For each instance, we introduce a latent binary variable $z\in\{0,1\}$ indicating whether a relationship between $\bm{x}$ and $y$ is present. We define the reliability score as $w(\bm{x})=p(z=1\mid\bm{x})$, which is assumed to be a smooth function of $\bm{x}$. By interpreting $z$ as an indicator of predictability, the conditional distribution of $y$ can be decomposed as
\begin{align}\label{eqn-1}
p(y\mid\bm{x})
&=p(y\mid\bm{x},z=1)w(\bm{x})
+p_{\mathrm{rand}}(y)[1-w(\bm{x})],
\end{align}
where $p_{\mathrm{rand}}(y):=p(y\mid\bm{x},z=0)$, since $y$ is independent of $\bm{x}$ when $z=0$. As illustrated in Figure~\ref{fig:01c}, the reliability score can be learned by a separate neural network in a principled manner, resulting in the learner network identifying the predictable component $p(y\mid\bm{x},z=1)$ instead. The proposed framework applies to both regression and classification, can be trained efficiently with probabilistic guarantees, and can be extended to incorporate supervised feature transformations that preserve locality.

Our main contributions are as follows:

\begin{enumerate}
\item We propose a probabilistic framework that explicitly decomposes the conditional distribution of the response into predictable and unpredictable components through an input-dependent reliability score. The reliability score quantifies the local presence of a predictive relationship between the features and response and is learned by a separate neural network. The proposed decomposition improves the interpretability of neural learners by explicitly modeling and learning both the reliability structure and the predictive component, while retaining the capability of a naive neural learner.

\item We develop a principled training procedure that combines the expectation--maximization (EM) algorithm \citep{dempster1977maximum} with stochastic gradient descent (SGD; \citet{rumelhart1986learning}) and maximum likelihood estimation (MLE) to jointly learn the reliability function and the predictive component. Experiments on both synthetic and real-world datasets demonstrate the effectiveness of the proposed model and training procedure in uncovering meaningful structures from the data that can facilitate decision-making and feature selection.

\item We show that the framework can be naturally extended with supervised feature transformation networks, allowing the learned representation to be optimized jointly with the final supervised task. When the feature--response relationship exhibits meaningful local structure, such as under the manifold hypothesis \citep{meilua2024manifold}, the framework can further split the problem into a collection of local models. Each local model can nevertheless be trained using the entire dataset, with observations appropriately weighted according to the corresponding probabilistic structure.
\end{enumerate}

For the remainder of the article, Section~\ref{background} introduces related work and relevant background, Section~\ref{design} describes the properties, training, and extensions of our model, Section~\ref{experiments} presents the experimental results, and Section~\ref{conclusion} concludes the article. Additional materials are provided in the appendix.

\section{Related Work}\label{background}

Our proposed framework is related to several existing approaches in statistics and machine learning. In particular, similarities arise from the use of latent mixture components, feature-dependent weighting, and iterative optimization. We review representative approaches from these areas and clarify how their objectives and interpretations differ from those of our framework.

\subsection{Noisy Labels}

Classification with noisy labels considers settings in which observed response labels are contaminated or corrupted. Robust estimation in the presence of such contamination has been studied since the development of robust statistics \citep{copas1988binary}, and this line of research has continued with the development of modern machine learning \citep{natarajan2013learning, scott2013classification}. In contrast, although heteroscedasticity is well studied in regression, it primarily concerns variation in the conditional variance of the response given $\bm{x}$, rather than variation in the predictability of the response itself.

Our framework adopts a different modeling perspective. Rather than interpreting poorly explained responses as corrupted observations, we allow the degree of predictability to vary across the feature space. An unpredictable response may therefore be perfectly valid, but simply not be explainable from the available features. Moreover, whereas existing approaches to noisy labels often rely on heuristic procedures \citep{han2018co} or prespecified parametric models of label corruption \citep{patrini2017making}, our framework uses neural networks to flexibly model reliability and incorporates the distinction between predictable and unpredictable components directly into a unified probabilistic model.

\subsection{Reliability and Feature-dependent Weights}

The reliability score $p(z=1\mid\bm{x})$ in our framework is a natural measure of the trustworthiness of the model prediction. Various approaches have been proposed in the past to quantify the predictability or reliability of a learner's output. These include measures such as the Bayes error \citep{ishida2022performance}, which characterizes the intrinsic predictability of a task, as well as neural networks trained to produce task-oriented selection or confidence scores, such as the selection function in SelectiveNet \citep{geifman2019selectivenet}. In contrast, our framework encodes reliability into the probabilistic model structure, and seeks to learn it directly.

As we will demonstrate, the (posterior) reliability score also serves as a weight for training the classification or regression network, allowing it to uncover the distribution of the predictable component within the overall distribution of $y$. Although assigning different weights to instances during training is common, such weights are often defined heuristically without probabilistic guarantees. The use of input-dependent weights is also superficially related to mixture models \citep{jacobs1991adaptive, huang2012mixture, xue2022machine}, where weights determine the contribution of different local predictive models based on the input. In our framework, however, the weights have a different meaning: we have one single model and the input-dependent reliability score represents the local probability that the response is predictable from the available features.

\subsection{EM-based Training with Neural Networks}

The training of neural networks embedded within probabilistic models has long been investigated, particularly in the Bayesian literature \citep{mackay1992practical, neal2012bayesian}. Since our model involves latent variables, its training should naturally combine the EM algorithm with SGD (with MLE also involved). Similar combinations have appeared in previous work \citep{ng2004using, lathuiliere2018deepgum}, although the specific procedures vary according to the underlying probabilistic models. Details of our training procedure are presented in Sections~\ref{training}.

\begin{figure}[htbp]
    \centering
    \includegraphics[width=\textwidth]{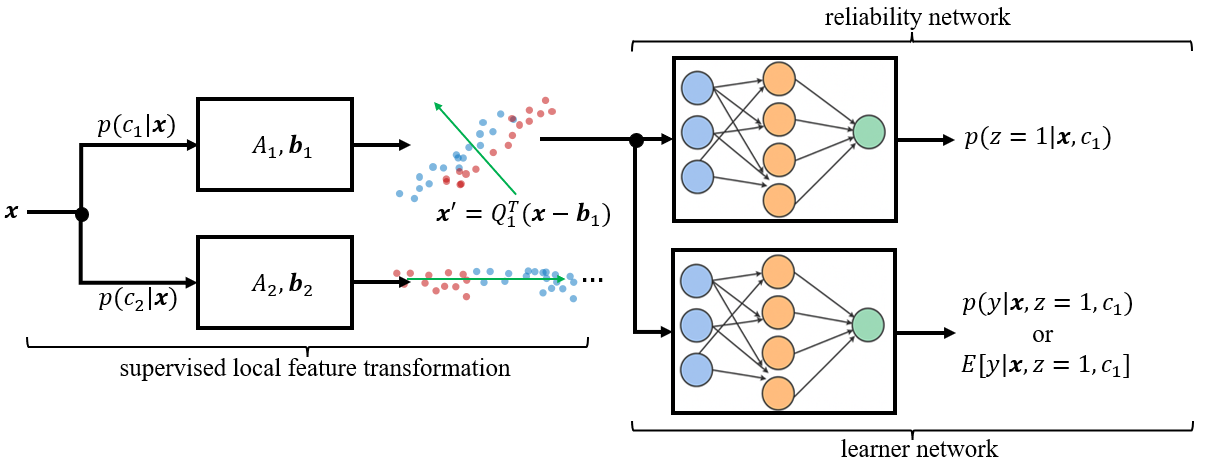}
    \caption{Overall architecture of our framework.}
    \label{fig:02}
\end{figure}

\section{Our Framework}\label{design}

In this section, we provide details of our framework, including the model specification and properties, training procedure, and its extension to enable localized feature transformations supervised by the target tasks.

\subsection{Basic Structure and Properties}

Figure~\ref{fig:02} illustrates the overall architecture of our framework, which, according to Equation~\ref{eqn-1}, seeks to model variation in predictability across the feature space. In addition to the neural network used for classification or regression (the learner), we employ a separate neural network to learn the reliability score $w(\bm{x}) = p(z=1\mid\bm{x})$. Throughout this work, we use MLPs with a single hidden layer of 64 neurons for both networks, leaving the exploration of more complex architectures for future work. We also recognize that different regions of the feature space may benefit from more carefully tuned local models, particularly in terms of feature transformation, the extension will be discussed in Section~\ref{feature}.

For a neural classifier, the conditional distributions of the response, $p(y\mid\bm{x},z=1)$ and $p_{\mathrm{rand}}(y)$, are simply class probabilities. For a neural regressor, we further model these conditional distributions as Gaussian, i.e.,
\begin{align}
y \mid \bm{x},z=1 &\sim N\left(y;f(\bm{x}),\sigma_{1}^2\right), \
y \mid \bm{x}, z=0 \sim N\left(y;\mu_0,\sigma_{0}^2\right),
\end{align}
where $f,\mu_0, \sigma_0^2$ and $\sigma_1^2$ are unknown functions and parameters. In other words, when $z=1$, we have $\mathbb{E}[y\mid\bm{x},z=1]=f(\bm{x})$, whereas when $z=0$, $y$ is independent of $\bm{x}$ and behaves like white noise.

Notice that in both cases, we have 
\begin{align} p(y, z \mid \bm{x}) &= p(z\mid\bm{x})p(y\mid \bm{x}, z) \nonumber\\ &= p(z = 1\mid\bm{x})^z[1 - p(z = 1\mid\bm{x})]^{1 - z}p(y\mid \bm{x}, z = 1)^zp_{\mathrm{rand}}(y)^{1 - z}, \end{align} hence, the joint log-density of $y$ and $z$ can be obtained as \begin{align}\label{eqn-joint} \log p(y, z\mid \bm{x}) &= \underbrace{ z\log p(z = 1\mid\bm{x}) +(1-z)\log[1-p(z = 1\mid\bm{x})] }_{\text{reliability network}} \nonumber\\ &\quad+ \underbrace{ z\log p(y\mid\bm{x},z=1) }_{\text{learner network}} + \underbrace{ (1-z)\log p_{\mathrm{rand}}(y) }_{\text{random-response component}}, \end{align} 
which is naturally decomposed into three components corresponding to the reliability network, the learner network, and the estimation of the distribution of the unpredictable component, respectively.

Ideally, according to Equation~\ref{eqn-joint}, if the latent variable $z$ is known, the learner network should be trained on instances with $z = 1$, which should lead to it learn the distribution of the predictable component $p(y\mid \bm{x}, z = 1)$. However, since $z$ is not observable, as we will demonstrate in Section~\ref{training}, the empirical learner target used in training is actually the log likelihood of the clean component weighted by the posterior reliability
\begin{align}\label{eqn-posterior}
\gamma(\bm{x}, y) :&= p(z = 1 \mid \bm{x}, y) \nonumber\\&= \frac{p(z=1\mid \bm{x})p(y \mid \bm{x}, z = 1)}{p(z=1\mid \bm{x})p(y \mid \bm{x}, z = 1) + p_\mathrm{rand}(y)[1 - p(z=1\mid \bm{x})]}.  
\end{align}
The good thing is, assuming that training converges to the optimum, this procedure also leads the learner to recover the predictable component of the response, which is precisely the distribution the learner is intended to model. See Appendix~\ref{appendix-a} for the proofs of the following results.

\begin{theorem}\label{thm:01}
Assume that the training procedure described in Section~\ref{training} is applied to our framework for regression and reliability estimation is exact. At the population optimum, the learner network outputs
$$
f^*(\bm{x}) = \mathbb{E}[y\mid\bm{x},z=1].
$$
\end{theorem}

\begin{theorem}\label{thm:02}
Assume that the training procedure described in Section~\ref{training} is applied to our framework for classification and reliability estimation is exact. At the population optimum, the learner network outputs the conditional distribution
$$
f^*(\bm{x}) = p(y\mid\bm{x},z=1),
$$
over the classes of $y$.
\end{theorem}

\subsection{Training Procedure}\label{training}

Direct maximization of the joint density in Equation~\ref{eqn-joint} is not possible because the latent variable $z$ is unobserved. We therefore employ the EM algorithm. At the E-step, the posterior reliability in Equation~\ref{eqn-posterior} is estimated and used to construct the expected log-likelihood, which is then maximized in the M-step. The reliability and learner networks are optimized using SGD, whereas the unpredictable component admits a closed-form maximization.

Let $\bm{\theta}^t = (\bm{\theta}_R^t, \bm{\theta}_L^t, \bm{\theta}_N^t)$ denote the model parameters at training step $t$, corresponding to the reliability network and the predictable and unpredictable components of the response distribution, respectively. 

In the E-step, we compute the expectation of the log-likelihood with respect to the posterior distribution of the latent variable $z$, yielding the $Q$-function
\begin{align}\label{eqn-qfunction}
Q(\bm{\theta} \mid \bm{\theta}^t)
:&=
\mathbb{E}_{p(z\mid \bm{x},y;\bm{\theta}^t)}
\left[
\log p(y,z\mid\bm{x};\bm{\theta})
\right] \nonumber\\
&=
\underbrace{\gamma^t
\log p(z=1\mid\bm{x};\bm{\theta}_R)
+(1-\gamma^t)
\log[1-p(z=1\mid\bm{x};\bm{\theta}_R)]}_{Q(\bm{\theta}_R\mid \bm{\theta}^t)}
\nonumber\\
&\quad
+\underbrace{\gamma^t
\log p(y\mid\bm{x},z=1;\bm{\theta}_L)}_{Q(\bm{\theta}_L\mid \bm{\theta}^t)}
+\underbrace{(1-\gamma^t)
\log p_{\mathrm{rand}}(y;\bm{\theta}_N)}_{Q(\bm{\theta}_N\mid \bm{\theta}^t)},
\end{align}
where $\gamma^t = p(z=1\mid\bm{x},y;\bm{\theta}^t)$ is the posterior reliability at training step $t$, obtained by substituting the current parameter estimates into Equation~\ref{eqn-posterior}.

In the M-step, the posterior reliabilities ${\gamma^t}$ are held fixed. Since the decomposition of the joint log-likelihood in Equation~\ref{eqn-joint} carries directly over to the corresponding $Q$-function, the parameters associated with the different components can be optimized separately. In particular, an SGD update is performed for both the reliability and learner networks:
\begin{align}
\bm{\theta}_R^{t+1}
&\leftarrow
\bm{\theta}_R^t
+
\eta
\left.
\nabla_{\bm{\theta}_R}
Q(\bm{\theta}_R\mid\bm{\theta}^t)
\right|_{\bm{\theta}_R=\bm{\theta}_R^t},
\qquad
\bm{\theta}_L^{t+1}
\leftarrow
\bm{\theta}_L^t
+
\eta
\left.
\nabla_{\bm{\theta}_L}
Q(\bm{\theta}_L\mid\bm{\theta}^t)
\right|_{\bm{\theta}_L=\bm{\theta}_L^t}.
\end{align}
The parameters of the unpredictable component, as well as $\sigma_1^2$  for the regression case, admit closed-form maximum likelihood estimates. Specifically, for regression, we have
\begin{align}\label{unpred-reg}
(\sigma_1^2)^{t+1}
=
\frac{
\sum_k \gamma_k^t
\left[
y_k-f(\bm{x}_k;\bm{\theta}_L^t)
\right]^2
}{
\sum_k \gamma_k^t
},
\
\mu_0^{t+1}
=\frac{
\sum_k (1-\gamma_k^t)y_k
}{
\sum_k (1-\gamma_k^t)
},
\
(\sigma_0^2)^{t+1}
=
\frac{
\sum_k (1-\gamma_k^t)
\left(y_k-\mu_0^{t+1}\right)^2
}{
\sum_k (1-\gamma_k^t)
};
\end{align}
and for classification, the unpredictable response distribution is updated as
\begin{align}\label{eqn-prandc}
p_{\mathrm{rand}}^{t+1}(y)
&=
\frac{
\sum_k (1-\gamma_k^t)
\mathbf{1}(y_k=y)
}{
\sum_k (1-\gamma_k^t)
}.
\end{align}

The E- and M-steps are repeated iteratively until convergence. And empirically, we find the following modification can be helpful.

\begin{remark}
    When applying the framework to classification, we find that directly using Equation~\ref{eqn-prandc} during iterative training can lead to model collapse. This may be due to the non-identifiability of the model and the strong interdependence between the estimation of the unpredictable component and the posterior reliability. To mitigate this issue, we use
    $$
    w_k^t = p(z=1\mid\bm{x}_k;\bm{\theta}^t)
    $$
    in place of \(\gamma_k^t\) when estimating the unpredictable component during iterative training. Equation~\ref{eqn-prandc} is applied only once after convergence to obtain the final estimate of \(p_{\mathrm{rand}}(y)\). Across all our experiments, this adapted procedure proves more stable and consistently yields meaningful decomposition.

\end{remark}

\begin{remark}
    The current optimization target is entirely probabilistic, which is sufficient when the reliability does not vary substantially over the feature space. When the reliability is expected to vary more delicately, an additional smoothness penalty can be incorporated into the M-step:
    $$
    \lambda_{\mathrm{smooth}}\sum_{k,\ell}s_{k\ell}\norm{w(\bm{x}_k)-w(\bm{x}_\ell)}^2,
    $$
    where \(s_{k\ell}\) is a predefined measure of similarity between \(\bm{x}_k\) and \(\bm{x}_\ell\), encouraging instances with similar features to have similar reliabilities.

\end{remark}

\subsection{Supervised Local Feature Transformation}\label{feature}

The original features $\bm{x}$ may not be optimal for the learner task. Instead of treating feature transformation as a preprocessing step optimized according to some intermediate objective, such as maximizing explained variance in principal component analysis (PCA), we can introduce a learnable feature transformation layer before both the reliability and learner networks\footnote{By doing so, we assume that the reliability and learner networks can share representations, although this can be readily relaxed.}. 

Let $A \in \mathbb{R}^{p\times d}$ be a randomly initialized full-rank matrix with QR decomposition $A=QR$, where $Q\in\mathbb{R}^{p\times d}$ has orthonormal columns and $R\in\mathbb{R}^{d\times d}$ is upper triangular, $\bm{b} \in \mathbb{R}^p$ be a translation vector. We then perform the feature transformation $\bm{x}' = Q^\top (\bm{x} - \bm{b})$, which projects the original features onto the learned $d$-dimensional subspace. The matrix $A$ and vector $\bm{b}$ can be trained jointly with the rest of the model as described in Section~\ref{training}, with the input of reliability and learner networks replaced by $\bm{x}'$. 

In some cases, a global feature transformation may not be well suited to our framework. Given the heterogeneity of the feature space, if predictability varies across different regions, as illustrated in Figure~\ref{fig:02}, the best representation of the features may also vary. 

To enable supervised local feature transformation, we assume that, for each instance, a prior distribution $p(c\mid\bm{x})$ is available, representing the probabilities that the instance belongs to each region $c$. Such a soft allocation can be obtained in various ways. For example, if the main structure of the dataset can be adequately described by a low-dimensional manifold, manifold learning via mixture models of variational autoencoders (VAEs) \citep{alberti2024manifold} can be used. Alternatively, simpler methods such as Gaussian mixture models \citep{reynolds2009gaussian} can be employed when a more direct partition of the feature space is appropriate.

With the soft allocation scheme, we can train separate feature transformation, reliability, and learner networks for different regions of the feature space. Specifically, treating $p(c\mid\bm{x})$ as a fixed external soft allocation, we define the allocation-weighted complete-data log-density as
\begin{align}\label{eqn-jointc}
\mathcal{L}_{\mathrm{soft}}(y,z\mid\bm{x})
&:=
\sum_c p(c\mid\bm{x})\log p(y,z\mid\bm{x},c)
\nonumber\\
&=
\underbrace{
z\sum_c p(c\mid\bm{x})\log p(z=1\mid\bm{x},c)
+(1-z)\sum_c p(c\mid\bm{x})
\log\big[1-p(z=1\mid\bm{x},c)\big]
}_{\text{reliability networks}}
\nonumber\\
&\quad+
\underbrace{
z\sum_c p(c\mid\bm{x})\log p(y\mid\bm{x},z=1,c)
}_{\text{learner networks}}
+
\underbrace{
(1-z)\log p_{\mathrm{rand}}(y)
}_{\text{random-response component}},
\end{align}
which is a lower bound of the exact log likelihood. Here, $p_{\mathrm{rand}}(y)$ is shared across local models, since the unpredictable component is assumed to be invariant across the feature space.

Training with $\mathcal{L}_{\mathrm{soft}}$ as the target has two main advantages. First, each local model can be trained using the full dataset, thereby avoiding the loss of effective sample size that would result from assigning instances exclusively to individual regions. Second, according to Equation~\ref{eqn-jointc}, the training procedure for each local model is essentially identical to the procedure in Section~\ref{training}, with the additional weighting factor \(p(c\mid\bm{x})\) accounting for the soft allocation.

The posterior reliability used in the E-step should likewise be aggregated across the local models. Specifically,
\begin{align}
\gamma(\bm{x},y)
&= p(z=1\mid\bm{x},y) \nonumber\\
&= \frac{
\sum_c p(y\mid\bm{x},z=1,c)
p(z=1\mid\bm{x},c)
p(c\mid\bm{x})
}{
\sum_c p(y\mid\bm{x},z=1,c)
p(z=1\mid\bm{x},c)
p(c\mid\bm{x})
+
\left[
1-\sum_c p(c\mid\bm{x})p(z=1\mid\bm{x},c)
\right]p_{\mathrm{rand}}(y)
}.
\end{align}
The unpredictable component \(p_{\mathrm{rand}}(y)\) is also updated collectively as in Equations~\ref{unpred-reg} and \ref{eqn-prandc} across all local models. For testing, predication from different local models can be aggregated as well probabilistically.

\section{Experiments}\label{experiments}

As discussed in Section~\ref{background}, the decomposition offered by our framework provides a unique perspective, for which appropriate evaluation and benchmarking remain open problems for future research. In this section, we conduct experiments on both simulated and real-world datasets seeking to establish three key aspects. First, since the components of our framework can be recombined in a principled manner to recover \(p(y\mid\bm{x})\) or \(\mathbb{E}[y\mid\bm{x}]\), we compare its predictive performance based on them with that of a naive learner to demonstrate that the decomposition does not compromise overall predictive power. Second, by applying our framework to a simulated dataset with a known data-generating mechanism, we demonstrate that the procedures described in Sections~\ref{training} and \ref{feature} behave as intended. Finally, experiments on a challenging financial dataset demonstrate that the additional structure revealed by our framework is meaningful and can offer practical benefits for downstream tasks such as decision-making and feature selection.

\subsection{Decomposition with a Simulated Dataset}

Figure~\ref{fig:03} illustrates the generating mechanism of a simulated dataset under our framework. Each instance has two features and a binary label. The feature space is equipped with a soft allocation scheme over two regions, with $p(c_1\mid\bm{x})$ gradually decreasing from 1 to 0 as $x_1$ moves from left to right. In each region, $p(z=1\mid\bm{x})$ is defined by a Gaussian-shaped function, such that reliability decreases with distance from the regional centre. The predictable component $p(y\mid\bm{x}, z=1, c)$ is defined by applying a sigmoid function to a quadratic function, yielding a clean decision boundary in each region in the form of a parabola. In the left region, the parabola is rotated by $45^\circ$, resulting in a different optimal representation basis. The reliability function is symmetric and therefore can work with any region-specific representation suitable for the classification task. The unpredictable component is homogeneous across the feature space but has an imbalanced label distribution, with probabilities 0.4 and 0.6 for labels 0 and 1, respectively.

\begin{figure}[H]
    \centering
    \begin{subfigure}[b]{0.32\textwidth}
        \centering
        \includegraphics[height=0.6\textwidth]{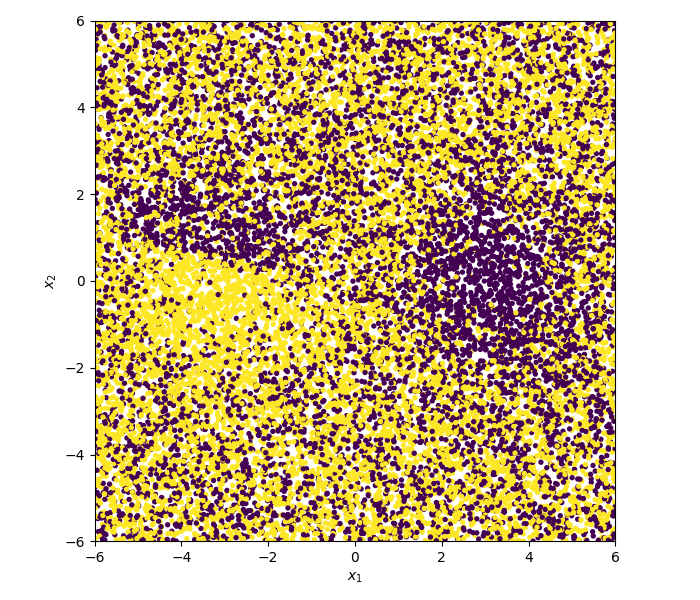}
        \caption{Labeled dataset.}
        \label{fig:03a}
    \end{subfigure}
    \hfill
    \begin{subfigure}[b]{0.32\textwidth}
        \centering
        \includegraphics[height=0.6\textwidth]{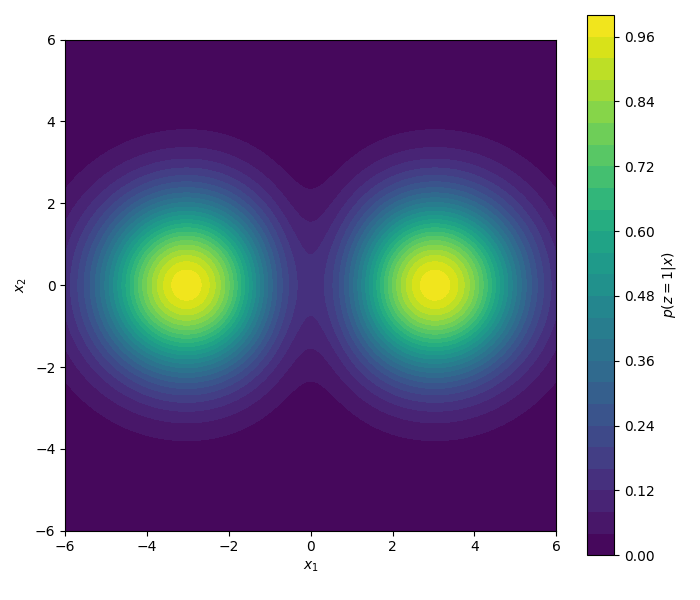}
        \caption{True $p(z = 1\mid \bm{x})$.}
        \label{fig:03b}
    \end{subfigure}
    \hfill
    \begin{subfigure}[b]{0.32\textwidth}
        \centering
        \includegraphics[height=0.6\textwidth]{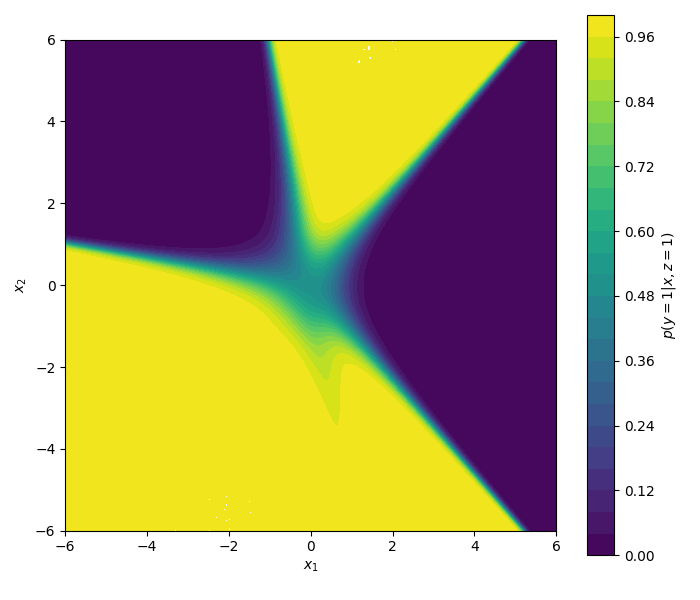}
        \caption{True $p(y = 1\mid\bm{x}, z = 1)$.}
        \label{fig:03c}
    \end{subfigure}

        \caption{(a) Labeled training dataset in feature space. (b) True reliability $p(z = 1\mid \bm{x})$. (c) True predictable component distribution $p(y = 1\mid\bm{x}, z = 1)$.}
    \label{fig:03}
\end{figure}

A vanilla neural classifier achieves an accuracy of $61\%$ on the test dataset. However, as illustrated by the results from our model in Figure~\ref{fig:04}, this single number omits substantial information. As shown in Figures~\ref{fig:04a}, \ref{fig:04b}, and \ref{fig:04c}, our model is able to correctly identify the reliability, predictable component, and suitable local feature transformations. It also estimates $p_{\mathrm{rand}}$ as $0.41$ and $0.59$ for the two classes, respectively, which is very close to the ground truth. 

In Figure~\ref{fig:04d}, we report the classification accuracy of the classification network alone, i.e., the predication is made according to the predictable component $p(y\mid\bm{x}, z=1)$, on instances whose reliability exceeds different thresholds. The central promise of our framework is that the classification network should achieve higher accuracy on more reliable instances. This is confirmed by the figure, where accuracy increases with the reliability threshold. Meanwhile, by combining the components and making predictions according to $p(y\mid\bm{x})$, our framework achieves an overall accuracy of $62\%$, essentially matching that of the vanilla neural classifier. Thus, the decomposition preserves overall predictive performance while providing additional insight into prediction reliability. Furthermore, as long as the classification network is applied to instances with reliability $\geq 0.1$, it can always achieve a higher accuracy.

\begin{figure}[H]
    \centering
    \begin{subfigure}[b]{0.24\textwidth}
        \centering
        \includegraphics[height=0.6\textwidth]{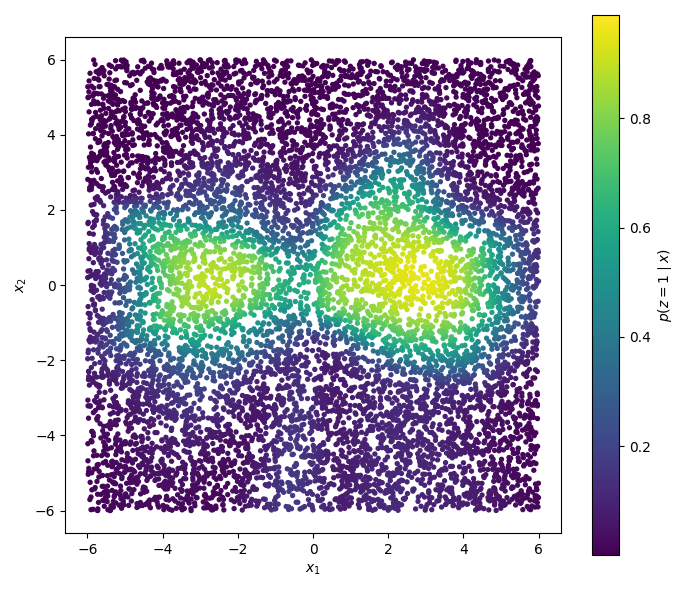}
        \caption{}
        \label{fig:04a}
    \end{subfigure}
    \hfill
    \begin{subfigure}[b]{0.24\textwidth}
        \centering
        \includegraphics[height=0.6\textwidth]{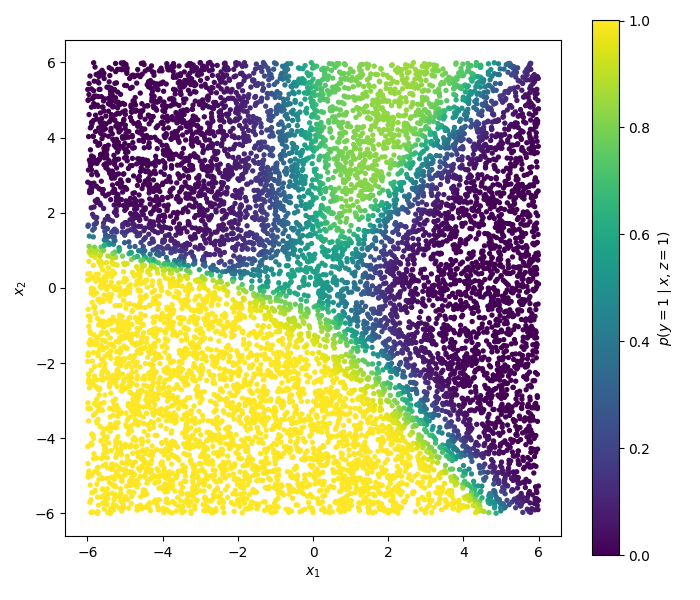}
        \caption{}
        \label{fig:04b}
    \end{subfigure}
    \hfill
    \begin{subfigure}[b]{0.24\textwidth}
        \centering
        \includegraphics[height=0.6\textwidth]{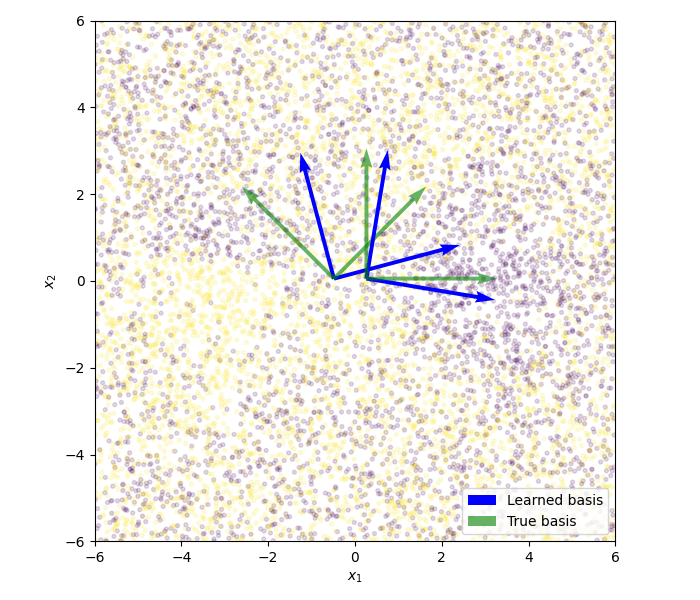}
        \caption{}
        \label{fig:04c}
    \end{subfigure}
    \hfill
    \begin{subfigure}[b]{0.24\textwidth}
        \centering
        \includegraphics[height=0.6\textwidth]{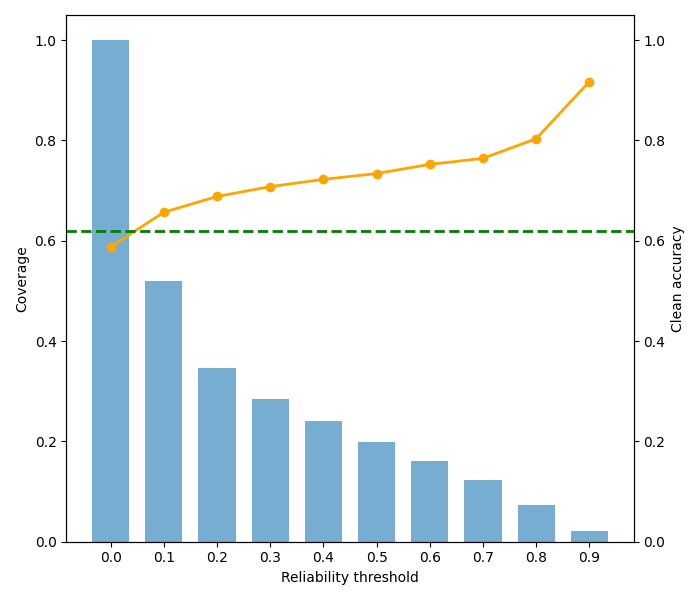}
        \caption{}
        \label{fig:04d}
    \end{subfigure}

        \caption{(a) Estimated reliability on test data. (b) Estimated predictable component on test data. (c) Theoretical (green) and learned (blue) local feature transformations. (d) Classification accuracy on test data with different reliability thresholds.}
    \label{fig:04}
\end{figure}


\subsection{Decomposition with a Real-world Dataset}

Directional prediction on financial data is generally considered a highly challenging task \citep{malkiel1985random}, with neural learners often generalizing poorly that limit their practical applicability \citep{ahmed2022artificial}. But poor overall predictive performance does not necessarily imply that the entire feature space is uninformative. As we will demonstrate, despite the inherent difficulty of the task, certain regions of the feature space can be more predictive than others, and many meaningful structures are hidden behind a single bad number.

For the real-world dataset, we consider stock 15-minute price data for JPMorgan Chase \& Co. (JPM), which include the open ($o$), close ($c$), and average ($a$) trading prices during each 15-minute interval, as well as the bid and ask prices at the beginning of the interval. Let $t$ denote the beginning timestamp of an interval. We construct three features as follows:
\begin{align*}
x_1^t &= \frac{a_{t-1} - o_{t-1}}{o_{t-1}}, \quad
x_2^t = \frac{c_{t-1} - a_{t-1}}{a_{t-1}}, \quad
x_3^t = \frac{\mathrm{ask}_t - \mathrm{bid}_t}{c_{t-1}}.
\end{align*}
For the classification task, we define a three-class response variable with labels $0$, $2$, and $1$, corresponding to the cases $a_t < \mathrm{bid}_t$, $a_t > \mathrm{ask}_t$, and $\mathrm{bid}_t \leq a_t \leq \mathrm{ask}_t$, respectively. For the regression task, a closely related response variable is defined as
\begin{align*}
y_t = \frac{a_t - (\mathrm{bid}_t + \mathrm{ask}_t)/2}{\mathrm{ask}_t - \mathrm{bid}_t}.
\end{align*}
The training set is then constructed with these features and responses using historic data from April 2017 to November 2024, and the test set comprises subsequent data till May 2026.

We apply our framework to the dataset in both classification and regression settings, standardizing the features beforehand and using a global feature transformation with $\bm{b}=\bm{0}$. As in the simulated-data experiments, the most important thing is to examine whether the decomposition recovers meaningful predictive structure in the feature space without sacrificing the predictive performance of a vanilla neural learner. Indeed, our framework achieves an overall classification accuracy of $44\%$ and a regression root mean squared error (RMSE) of $5.03$, closely matching the corresponding results of vanilla neural learners ($44\%$ and $5.02$, respectively). Meanwhile, as illustrated in Figures \ref{fig:05a} and \ref{fig:05d}, predictability varies across the feature space: instances in regions of higher estimated reliability tend to exhibit higher classification accuracy and lower prediction RMSE.

We also observe that the predictability structure appears more pronounced in classification than in regression, as reflected in the greater variation in the proportion of instances covered across reliability levels. One possible explanation is that predicting direction is easier than predicting magnitude, which requires both directional and quantitative accuracy. Furthermore, Figure~\ref{fig:05a} shows that the classifier achieves higher accuracy than its overall accuracy only at very high reliability levels ($w \geq 0.9$), suggesting that the estimated unpredictable component can provide complementary predictive information for classification. In contrast, Figure~\ref{fig:05d} shows that the regression network achieves an RMSE below the overall RMSE at every reliability level, suggesting that incorporating the estimated unpredictable component may be detrimental to regression performance. This difference may be explained by the substantially larger estimated noise scale of the unpredictable component ($\hat{\sigma}_0 = 17.01$) compared with that of the predictable component ($\hat{\sigma}_1 = 5.80$), which increases the variance of the combined predictive distribution and can consequently inflate its RMSE.

\begin{figure}[H]
    \centering
    \begin{subfigure}[b]{0.24\textwidth}
        \centering
        \includegraphics[height=0.6\textwidth]{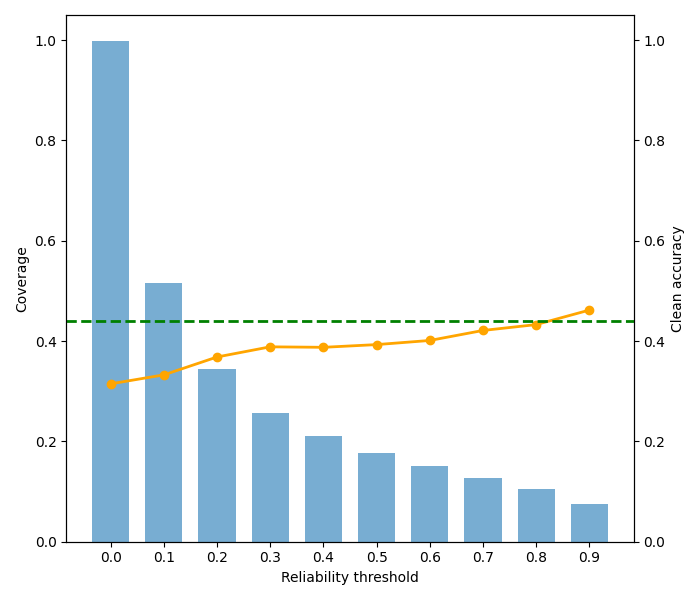}
        \caption{}
        \label{fig:05a}
    \end{subfigure}
    \hfill
    \begin{subfigure}[b]{0.24\textwidth}
        \centering
        \includegraphics[height=0.6\textwidth]{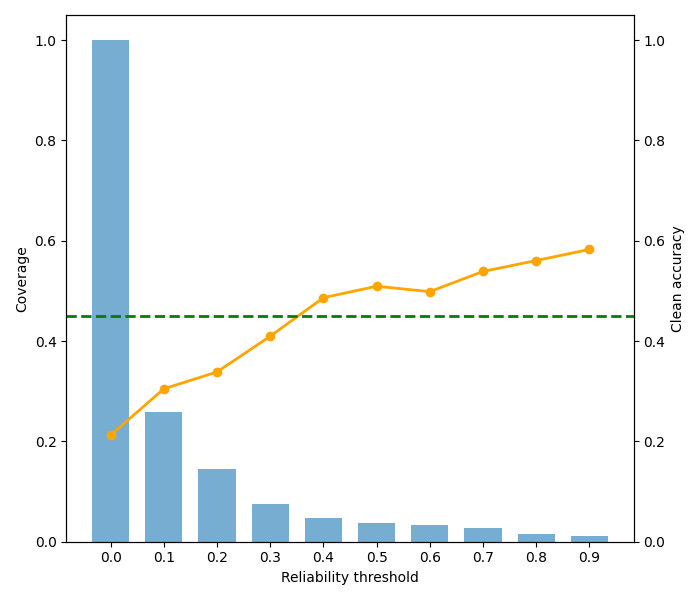}
        \caption{}
        \label{fig:05b}
    \end{subfigure}
    \hfill
    \begin{subfigure}[b]{0.24\textwidth}
        \centering
        \includegraphics[height=0.6\textwidth]{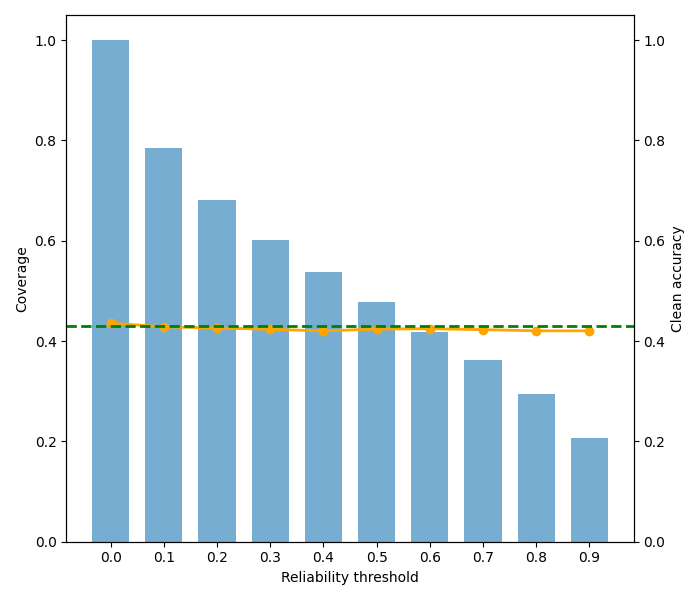}
        \caption{}
        \label{fig:05c}
    \end{subfigure}
    \hfill
    \begin{subfigure}[b]{0.24\textwidth}
        \centering
        \includegraphics[height=0.6\textwidth]{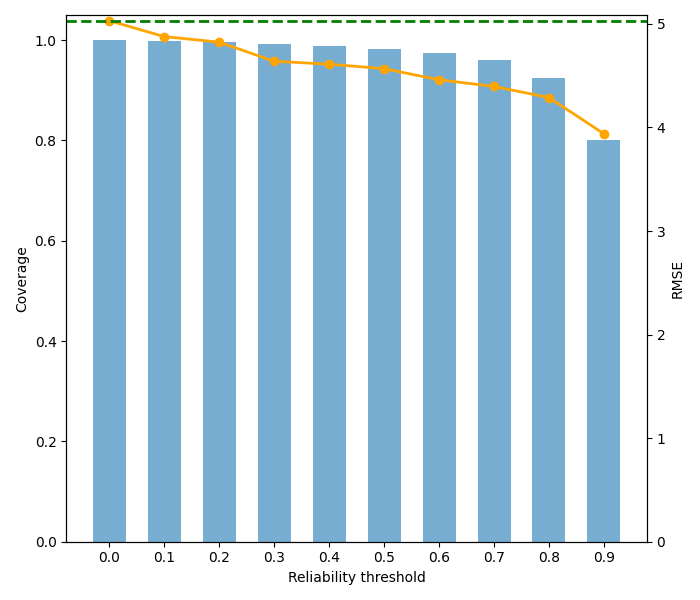}
        \caption{}
        \label{fig:05d}
    \end{subfigure}

    \caption{(a) Classification accuracy with different reliability thresholds. (b) Classification accuracy versus reliability thresholds with $x_2$ removed. (c) Classification accuracy versus reliability thresholds with $x_4$ added. (b) Predictation RMSE with different reliability thresholds.}
    \label{fig:05}
\end{figure}

For practitioners, the additional structure provided by our framework can inform practical decision-making in several ways. For a fixed dataset, the predictable component and its associated reliability should remain the primary basis for decision-making, since the unpredictable component captures aspects of the response that cannot be explained by the available features. As our results demonstrate, incorporating this component into predictions may or may not improve performance, and its usefulness cannot be determined a priori. Furthermore, our framework provides a principled way to evaluate the effects of adding or removing features by examining changes in the estimated reliability structure and the predictive performance of the predictable component.

As a simple illustration, we consider two additional cases: removing feature \(x_2\) and adding the feature

$$
x_4^t = \frac{c_{t-1} - o_{t-1}}{o_{t-1}}.
$$

Both cases achieve overall classification accuracies comparable to that of the baseline. However, Figure~\ref{fig:05b} suggests that removing \(x_2\) produces a more pronounced predictability structure: the reliability network becomes more selective, and instances assigned higher reliability are classified more accurately. In contrast, Figure~\ref{fig:05c} suggests that adding \(x_4\) substantially weakens the reliability structure, as reliability thresholds become less effective at distinguishing instances with different levels of predictability, and high-reliability instances no longer exhibit correspondingly high classification accuracy. This may indicate that \(x_4\) provides largely redundant information given \(x_1\) and \(x_2\), consistent with the definitions of these features. These observations suggest that changes in the reliability structure, considered alongside the predictive performance of the predictable component, could provide a basis for developing a more systematic feature selection procedure.

\section{Conclusion}\label{conclusion}

In this paper, we propose a probability principled way to decompose a neural learner, which, on top of its predictive performance, allows it to discover the variation of preidtcability across the feature space. The resulting structure is the split of predictive and unpredictive component of the response. A training procedure is designed which can also be easily extended to allow supervised local feature transformations. We proved that the optimal training result will indeed lead to the learner network recover the predictable component of the response, and experiments on both synthetic and real datasets shows the effectiveness of our proposals, which reveals more insight compared to their naive counterpart.

For future work, it would be desirable to develop systematic methods for selecting various hyperparameters, such as the hidden layer size, feature transformation dimensions, and the number of local regions. We believe that classical statistical tests, as well as empirically motivated engineering algorithms, could be developed for this purpose. More sophisticated architectures for the reliability and learner networks may further improve the framework, but their choice should be guided by the underlying data structure that the framework seeks to capture.

\section{Limitation}\label{conclusion}

Our framework requires all features to be continuous, and for the neural networks to successfully identify the decomposition, different components of the model need to vary smoothly as well. Abrupt changes in reliability or in the predictable component can lead to poor estimation. Also, although Theorems~\ref{thm:01} and \ref{thm:02} establish that, when the estimated posterior reliability is exact, the population optimum corresponds to the predictable component of the response variable, our training procedure does not guarantee accurate estimation of the reliability or convergence to the global population optimum. In fact, the convergence of the combined EM, SGD, and MLE procedure is itself not theoretically established and remains a challenging problem in general. These limitations are further compounded by the identifiability issues commonly associated with latent variables and highly parameterized neural network architectures.

\bibliographystyle{plainnat}
\bibliography{reference}

\appendix
\section{Proofs of Theorems \ref{thm:01} and \ref{thm:02}}\label{appendix-a}

In this section, we provide proofs for Theorems \ref{thm:01} and \ref{thm:02} from Section \ref{design}.

\addtocounter{theorem}{-2}

\begin{theorem}
Assume that the training procedure described in Section~\ref{training} is applied to our framework for regression and reliability estimation is exact. At the population optimum, the learner network outputs
$$
f^*(\bm{x}) = \mathbb{E}[y\mid\bm{x},z=1].
$$
\end{theorem}

\begin{proof}

Denote the output of the learner network for regression by $f(\bm{x};\bm{\theta}_L)$, which is intended to estimate the conditional mean of the predictable component in the response. Our target is to show that the training procedure described in Section~\ref{training} indeed leads it to converge to the true conditional mean at the population optimum. 

According to the training procedure, the learner network is trained by maximizing the log likelihood of the predicable component of the response, $\log p(y\mid \bm{x}, z = 1)$, weighted by the posterior reliability, which is also the conditional mean of the latent variable $z$, i.e.,
$\gamma(\bm{x},y) = p(z=1\mid\bm{x},y) = \mathbb{E}[z\mid\bm{x},y]$, hence, at the instance level, the optimization target is
$\mathbb{E}\left[z\log p(y\mid\bm{x},z=1;\bm{\theta}_L)\mid \bm{x}, y\right],$
where the expectation is taken with respect to $p(z \mid \bm{x}, y)$,  and when we iterate over the entire dataset, the population level optimization target is
\begin{align}
    \mathcal{L}_{\mathrm{reg}}
    = \mathbb{E}
    \left[
        \mathbb{E}
        \left[
            z
            \log
            p(y\mid\bm{x},z=1;\bm{\theta}_L)
            \,\middle|\,
            \bm{x},y
        \right]
    \right],
\end{align}
where the outer expectation is taken with respect to the joint distribution of $\bm{x}$ and $y$.

Since $z\in\{0,1\}$, we further have
\begin{align}
    \mathcal{L}_{\mathrm{reg}}
    &=
    \mathbb{E}
    \left[
        \mathbb{E}
        \left[
            z
            \log
            p(y\mid\bm{x},z=1;\bm{\theta}_L)
            \,\middle|\,
            \bm{x},y
        \right]
    \right]
    \nonumber\\
    &=
    \iint
    p(z=1)p(\bm{x},y\mid z = 1)
    \log
    p(y\mid\bm{x},z=1;\bm{\theta}_L)d\bm{x}dy
    \nonumber\\
    &=
    p(z=1)
    \mathbb{E}
    \left[
        \log
        p(y\mid\bm{x},z=1;\bm{\theta}_L)
        \mid z=1
    \right]\nonumber\\
    &= p(z=1)
    \mathbb{E}\left[\mathbb{E}
    \left[
        \log
        p(y\mid\bm{x},z=1;\bm{\theta}_L)
        \mid \bm{x}, z=1
    \right]\mid z=1 \right],
\end{align}
where we use the law of total expectation in the last equality. This means that, at the population level, the objective is determined entirely by the $z = 1$ instances, as desired.

Since $z$ is independent of $\bm{\theta}_L$, maximizing
$\mathcal{L}_{\mathrm{reg}}$ is equivalent to maximizing the inner expectation, under our Gaussian assumption, it is known in closed form as
\begin{align}
    -\mathbb{E}\left[\frac{1}{2\sigma_1^2}
        \left\{
        y-f(\bm{x};\bm{\theta}_L)
        \right\}^2\mid \bm{x}, z = 1\right]
    -
    \frac{1}{2}\log(2\pi\sigma_1^2).
\end{align}
which is maximized when
\begin{equation}
        f^*(\bm{x})
        =
        \mathbb{E}[y\mid\bm{x},z=1].
\end{equation}
Thus, at the population optimum, the regression network recovers the
conditional mean of the predictable component.
\end{proof}

\begin{theorem}
Assume that the training procedure described in Section~\ref{training} is applied to our framework for classification and reliability estimation is exact. At the population optimum, the learner network outputs the conditional distribution
$$
f^*(\bm{x}) = p(y\mid\bm{x},z=1),
$$
over the classes of $y$.

\end{theorem}

\begin{proof}

    Denote the output of the learner network for classification as $p(y\mid\bm{x}, z = 1;\bm{\theta}_L)$, which seeks to estimate the class probabilities of the predictable component of the response. Our target is to show that the training procedure described in Section~\ref{training} indeed leads it to converge to the true conditional probabilities at the population optimum. 
    
    Similar to the regression case, the population level optimization target is
    \begin{align}
        \mathcal{L}_{\mathrm{cls}}
        &= 
        \mathbb{E}
        \left[
        p(z=1\mid\bm{x},y)
        \log p(y\mid\bm{x}, z = 1;\bm{\theta}_L)
        \right] \nonumber\\&=
        p(z=1)
        \mathbb{E}
        \left[\mathbb{E}
        \left[
        \log p(y\mid\bm{x}, z = 1;\bm{\theta}_L)\mid\bm{x}, z = 1
        \right]z = 1\right],
    \end{align}
    which means maximizing the population objective is equivalent to maximizing the inner expectation
    \begin{equation}
    \mathbb{E}
        \left[
        \log p(y\mid\bm{x}, z = 1;\bm{\theta}_L)\mid\bm{x}, z = 1
        \right] =\sum_y p(y\mid\bm{x},z=1)
    \log p(y\mid\bm{x}, z = 1;\bm{\theta}_c), 
    \end{equation}
    which is just the cross-entropy between the true predictable response distribution $p(y\mid\bm{x},z=1)$ and the classifier output $p(y\mid\bm{x};\bm{\theta}_c)$. Therefore, the population optimum is achieved when
    \[
    p(y\mid\bm{x}, z = 1;\bm{\theta}_c)
    =
    p(y\mid\bm{x},z=1),
    \]
    indicating that the classifier learns the predictable component of the label distribution.
\end{proof}

\end{document}